\documentclass[letterpaper]{article} 
\usepackage{aaai2026}  
\usepackage{times}  
\usepackage{helvet}  
\usepackage{courier}  
\usepackage[hyphens]{url}  
\usepackage{graphicx} 
\usepackage{natbib}  
\usepackage{caption} 
\usepackage{algorithm}
\usepackage{algorithmic}

\newcommand{\removed}[1]{}
\usepackage{times}
\usepackage{soul}
\usepackage{url}
\usepackage[utf8]{inputenc}
\usepackage{graphicx}
\usepackage{amsmath}
\usepackage{amsthm}
\usepackage{booktabs}
\usepackage{algorithm}
\usepackage{algorithmic}
\usepackage{amsmath}
\usepackage{pgfplots}

\usepackage{amsmath,bm}

\usepackage{algorithm}
\usepackage{algorithmic}

\usepackage{times}

\usepackage{url}

\usepackage{amsmath}

\usepackage{amssymb}
\usepackage{mathtools}
\usepackage{amsthm}

\usepackage{pgfplots}
\usepackage{tikz}
\usetikzlibrary{fillbetween}
\pgfplotsset{compat=1.18}

\usepackage{algorithmic}

\theoremstyle{plain}

\usepackage[textsize=tiny]{todonotes}
\usepackage{multirow}

\usepackage{ascmac}
\usepackage{float}
\usepackage{perpage}
\MakeSorted{figure}
\MakeSorted{table}

\usepackage{url}
\usepackage{natbib}
\usepackage{chapterbib}

\usepackage{color}
\usepackage{tikz}
\tikzset{%
mynode/.style={circle,minimum width=.5ex, fill=none,draw}, 
myfillnode/.style={circle,minimum width=.5ex, fill=lightgray,draw}, 
}
\usepackage{amssymb}
\usepackage{natbib}

\newcommand{\indep}{\perp \!\!\! \perp}
\usepackage{amsmath}               
\usepackage{lscape}
\usepackage{algorithm}
\usepackage{color}
\usepackage{tikz}
\usepackage{amsmath,amsthm}
\newtheorem{theorem}{Theorem}
\newtheorem{definition}{Definition}
\newtheorem{assumption}{Assumption}
\newtheorem{lemma}{Lemma}
\newtheorem{proposition}{Proposition}

\usepackage{stackengine}
\def\defeq{\mathrel{\ensurestackMath{\stackon[1pt]{=}{\scriptscriptstyle\Delta}}}}

\pgfplotsset{compat=1.18}
\usepackage{comment}
\usepackage{here}
\allowdisplaybreaks[4]
\usepackage{caption}
\usepackage{bbding}
\usepackage{arydshln}
\usepackage{afterpage}

\usepackage{mathrsfs}

\usetikzlibrary{fillbetween}
\usepackage{pgfplots}
\pgfplotsset{compat=1.18}

\usepackage{xcolor}
\newcommand{\jin}[1]{\textcolor{blue}{[[#1]]}}
\newcommand{\jina}[1]{\textcolor{blue}{#1}}
\newcommand{\yuta}[1]{\textcolor{red}{#1}}

\usepackage{soul}

\usepackage{algorithm}
\usepackage{algorithmic}

\usepackage{newfloat}
\usepackage{listings}
\DeclareCaptionStyle{ruled}{labelfont=normalfont,labelsep=colon,strut=off} 
\floatstyle{ruled}
\newfloat{listing}{tb}{lst}{}
\floatname{listing}{Listing}
\usepackage{newfloat}
\usepackage{listings}
\DeclareCaptionStyle{ruled}{labelfont=normalfont,labelsep=colon,strut=off} 
\floatstyle{ruled}
\newfloat{listing}{tb}{lst}{}
\floatname{listing}{Listing}
\title{Potential Outcome Rankings for Counterfactual Decision Making}
\author{
Yuta Kawakami, Jin Tian
}
\affiliations{
Mohamed bin Zayed University of Artificial Intelligence, UAE\\
\{Yuta.Kawakami, Jin.Tian\}@mbzuai.ac.ae
}

\usepackage{bibentry}

\begin{document}

\maketitle

\begin{abstract}
Counterfactual decision-making in the face of uncertainty involves selecting the optimal action from several alternatives using causal reasoning. Decision-makers often rank expected potential outcomes (or their corresponding utility and desirability) to compare the preferences of candidate actions. In this paper, we study new counterfactual decision-making rules by introducing two new metrics: the probabilities of potential outcome ranking (PoR) and the probability of achieving the best potential outcome (PoB). PoR reveals the most probable ranking of potential outcomes for an individual, and PoB indicates the action most likely to yield the top-ranked outcome for an individual. We then establish identification theorems and derive bounds for these metrics, and present estimation methods. Finally, we perform numerical experiments to illustrate the finite-sample properties of the estimators and demonstrate their application to a real-world dataset.
\end{abstract}



\section{Introduction}

Decision-making in the face of uncertainty involves selecting the optimal action from several alternatives $\{1,\dots,K\}$ using logical reasoning or empirical evidence.
Decision theory traces its origins to the work of Blaise Pascal, particularly Pascal’s Wager and his foundational contributions to probability theory \citep{Pascal1670}.
This line of reasoning aligns with the principles of expected utility theory, which later became a cornerstone of modern decision theory \citep{Bernoulli1954, Jeffrey1990, Knight2002}.
Today, AI is reshaping decision-making processes across a wide range of domains—from healthcare and finance to business strategy and everyday personal choices \citep{Lai2023}.
Causal inference, on the other hand, examines the cause-and-effect relationship between an action (or decision) ($X$) and its outcome (or corresponding utility, desirability) ($Y$) with potential outcomes (POs) $Y_x$ \citep{Pearl09}. 
When these two areas intersect, counterfactual decision-making emerges as a method for determining optimal action by integrating causal reasoning techniques \citep{Lewis1976,Gibbard1981,Pearl09,Dawid2021}.

Using causal reasoning, decision-makers often focus on evaluating \emph{the rank of the expected POs} (RoE),
given by the following counterfactual relationship:
\begin{equation}
\label{eq1}
\mathbb{E}[Y_{r_1}] > \mathbb{E}[Y_{r_2}] > \dots > \mathbb{E}[Y_{r_K}].
\end{equation}
Causal reasoning enables us to assess the expected POs, when they cannot be directly observed or evaluated. Based on the RoE in (\ref{eq1}), decision-makers judge $r_1$ as the best action and $r_K$ as the least favorable action. 
Many approaches rely on RoE-based decision-making, including methods in statistics \citep{STUDENT1908,Bartholomew1959, Page1963, McCaffrey2013}, in causal inference \citep{Lopez2017}, and reinforcement learning 
\citep{Sutton1998,Auer2002}.

Despite the widespread use of RoE-based decision-making, the action that maximizes the expected PO does not always maximize the PO for each individual. Similarly,  the action with the lowest expected PO is not necessarily the least favorable choice for an individual.
RoE is not a perfect decision-making criterion and may lead to undesirable choices, depending on the preferences of the decision-maker. 
To accommodate a broader range of decision-making preferences, alternative rules are desired.

In economics, discrete choice models (DCMs) offer an alternative framework for modeling \emph{individual} choice behavior \citep{Mcfadden1972, Ben1985, Anderson1992, Mcfadden2001, Train2009}.
{DCMs are used to describe and predict the probability that an individual selects a particular alternative from available options.  
The utility $Y(k)$  that a person obtains by choosing the alternative $k$  
is  modeled as a function of observable characteristics of the alternatives and the decision-maker, plus an unobservable random component. The core assumption is that individuals are rational and choose the alternative that maximizes their perceived utility. Researchers  typically evaluate the probability that a specific alternative, say $r_1$, is chosen, which can be expressed as:
\begin{equation}
\mathbb{P}\left(\arg\max\nolimits_{k=1,\dots,K}\{Y(k)\}=r_1\right).
\end{equation}
Sometimes the interest lies in the probability of a specific ranking of alternatives $(r_1,\dots,r_K)$, i.e.,
\begin{equation}
\mathbb{P}\left(Y({r_1})> Y({r_2})> \dots> Y({r_K})\right).
\end{equation}
It is important to clarify that $Y(k)$ (utility) in this context does not represent POs in the causal inference sense, and DCMs do not inherently constitute counterfactual reasoning. DCMs aim to model observed choices based on prevailing conditions, 
and to forecast what choices they will make if those conditions persist.  
}

Inspired by the DCM framework, we formulate new types of counterfactual decision-making rules using POs $Y_k$ and study their reasoning using historical experimental data.
{We  ask the following counterfactual questions for individualized decision-making: 
\begin{center}
({\bf Question 1}). ``\emph{What is the most probable rank\\ of POs for an individual?}"
\end{center}
\begin{center}
({\bf Question 2}). 
``\emph{Which action is most likely to yield\\ the best outcome for an individual?}"
\end{center}
To address Question 1, we propose the following \emph{probability of potential outcome ranking} (PoR) metric for ranks of actions $\bm{r}=(r_1,\dots,r_K)$: 
\begin{equation}
S_{PoR}(\bm{r}) \defeq \mathbb{P}\left(Y_{r_1}> Y_{r_2}> \dots> Y_{r_K}\right).
\end{equation}
To address Question 2, we propose the following 
\emph{probability of achieving the best potential outcome} (PoB) metric for actions $r_1\in \{1,\dots,K\}$:
\begin{equation}
S_{PoB}(r_1) \defeq \mathbb{P}\left(\arg\max\nolimits_{k=1,\dots,K}\{Y_k\}=r_1\right).
\end{equation}
Questions 1 and 2 are answered by optimizing the PoR and PoB metrics, respectively.
}

{While the standard RoE metric $S_{RoE}(k) = \mathbb{E}[Y_{k}]$ evaluates an action based on the population average of its POs, our proposed metrics differ fundamentally. The PoB metric is based on the proportion of individuals for whom the action yields the best PO. The PoR metric does not directly evaluate individual actions, but assesses the rankings of all possible actions based on the proportion of individuals exhibiting the corresponding order of POs. PoR allows for the identification of the most probable rank of POs, and one could then select the best action as the one ordered first in this rank.  PoR enables more nuanced decisions, particularly useful when the sequence or relative performance of actions matters, beyond just identifying a single optimal one. 
}

{Evaluating the PoR and PoB metrics poses significantly greater challenges than evaluating RoE, as it requires counterfactual reasoning that lies in the third layer of Pearl’s causal hierarchy \citep{Bareinboim2022,Pearl2018}, whereas RoE reasoning belongs to the second layer. This paper offers three key contributions: we establish identification theorems for PoR and PoB under the widely used rank invariance assumption \citep{Heckman1997}; we derive bounds for these metrics in the absence of this assumption;  and we present corresponding estimation methods.} 
Finally, we conduct numerical experiments to demonstrate the 
properties of the estimators and present an application using a real-world dataset.


\section{Illustrating Example}

We provide an illustrative example that highlights the differences between the three metrics RoE, PoR, and PoB. 

\textbf{Educational Scenario.}
In educational research, the causal effect of education plays a critical role practically. 
Specifically, imagine three classes (A, B, and C) designed to enhance students’ overall knowledge proficiency. 
Decision-makers, such as students, teachers, or parents, need to choose one course from among them.
Table \ref{tab:my_label} presents an example of the students’ potential scores.
Traditionally, decision-makers use the RoE metric and compare the expected scores across the available classes. 
Class A exhibits the highest expected score, followed by Class B. 
Based on this comparison, one concludes that Class A is the most beneficial, Class B is a bit less effective, and Class C is the least effective. 
However, 
one may reasonably conclude that the optimal ranking is Class C first, followed by Class B, and Class A last, based on the observation that this is the most probable rank (37.5\% chance). 
This conclusion corresponds to the PoR metric. 
Finally, some other decision-makers might prefer to enroll in Class B based on the PoB metric,  since B is  most likely to yield the highest score (in 50\% of the students). 


\begin{table}[tb]
    \centering
    \scalebox{0.8}{
    \begin{tabular}{c|ccc|c}
    \hline
      Student ID &  A &  B & C & Rankings\\
       \hline
       1  &  30 &40& \textbf{50} &A $<$ B $<$ C \\
       2 & 40 & 50 &\textbf{60} &A $<$ B $<$ C \\
       3& 50 & 60 & \textbf{70} & A $<$ B $<$ C \\
       4  &  30 &\textbf{50}& 40 &A $<$ C $<$ B \\
       5 & 40 & \textbf{60} & 50 &A $<$ C $<$ B \\
       6& 60 & \textbf{70} & 40 & C $<$ A $<$ B \\
       7& 50 & \textbf{60} & 40 & C $<$ A $<$ B \\
       8& \textbf{100} & 5 & 0 & C $<$ B $<$ A\\
       \hline
      Average of potential score & \textbf{50} & 49.375 & 43.75 & -\\
      Proportion of being the best score & 0.125 & \textbf{0.5} & 0.375 & - \\
\hline
\end{tabular}
}
\caption{An example of students' potential scores. We highlight the maximum value in each row.}
\label{tab:my_label}
\end{table}


As illustrated in this example, our PoB and PoR criteria yield different decision-making rules, potentially leading decision-makers to choose different actions.
Similar situations arise in medical contexts (such as selecting a type of surgery to improve patient life expectancy), in economic contexts (such as choosing a policy to improve economic performance in each municipality), and in marketing contexts (such as selecting an advertising strategy to influence the attitudes of customers).

\section{Notations and Backgrounds}

We represent each variable with a capital letter $(X)$ and its realized value with a small letter $(x)$.
Let $\mathbb{I}(x)$ be an indicator function that takes $1$ if $x$ is true and $0$ if $x$ is false.
Denote $\Omega_Y$ be the domain of $Y$,
$\mathbb{E}[Y]$ be the expectation of $Y$, and $F_Y(y)=\mathbb{P}(Y\leq y)$ be the cumulative distribution function (CDF) of a continuous variable $Y$. 
The inverse of the CDF is defined as $F^{-1}_Y(p)=\inf\{y;F_Y(y)\geq p\}$ for $p \in [0,1]$.

\textbf{Structural causal models (SCM).} 
We use the language of SCMs as our basic semantic and inferential framework \citep{Pearl09}.
An SCM ${\cal M}$ is a tuple $<{\boldsymbol U},{\boldsymbol V}, {\cal F}, \mathbb{P}_{\boldsymbol U}>$, where ${\boldsymbol U}$ is a set of exogenous (unobserved) variables following a joint distribution $\mathbb{P}_{\boldsymbol U}$, and ${\boldsymbol V}$ is a set of endogenous (observable) variables whose values are determined by structural functions ${\cal F}=\{f_{V_i}\}_{V_i \in {\boldsymbol V}}$ such that $v_i:= f_{V_i}({\mathbf{pa}}_{V_i},{\boldsymbol u}_{V_i})$ where ${\mathbf{PA}}_{V_i} \subseteq {\boldsymbol V}$ and $U_{V_i} \subseteq {\boldsymbol U}$. 
An atomic intervention of setting a set of endogenous variables ${\boldsymbol X}$ to constants ${\boldsymbol x}$, denoted by $do({\boldsymbol x})$, replaces the original equations of ${\boldsymbol X}$ by ${\boldsymbol X} :={\boldsymbol x}$ and induces a \textit{sub-model}  ${\cal M}_{\boldsymbol x}$.
We denote the potential outcome (PO) $Y$ under intervention $do({x})$ by $Y_{{x}}({\boldsymbol u})$, which is a solution of $Y$ given ${\boldsymbol U}={\boldsymbol u}$ in the sub-model ${\cal M}_{x}$.


\textbf{Identification of Counterfactuals.}
Researchers often consider the following SCM, ${\cal M}$:
\begin{equation}
Y:=f_Y(X,U^Y), X:=f_X(U^X),
\end{equation}
where $U^Y$ and $U^X$ are latent exogenous variables.
$Y$ is a continuous variable and $X$ is a discrete variable  taking values in $\Omega_X=\{1,2,\dots,K\}$. 
We assume the following: 
\begin{assumption}[Continuity and Differentiability of CDF of Potential Outcome]
\label{CON}
$F_{Y_x}(y)$ is continuous and differentiable for any $x \in \Omega_X$ and $y \in \Omega_Y$.
\end{assumption}
\begin{assumption}[Exogeneity]
\label{ASEXO2}
$Y_x\indep X$ for any $x \in\Omega_X$.
\end{assumption}
\noindent {\textbf{Remark.} Assumption \ref{ASEXO2} is used to identify  $\mathbb{P}(Y_x)$, under which we have $\mathbb{P}(Y_x) = P(Y|x)$. Alternatively, we may assume \emph{conditional exogeneity} given covariates $W$, i.e., $Y_x \indep X|W$ for any $x \in \Omega_X$, in which case $\mathbb{P}(Y_x) = \int P(Y|x,w)\mathfrak{p}(w)dw$, where $\mathfrak{p}$ denotes probability density function. Another possibility is that the interventional distribution $\mathbb{P}(Y_x)$ is directly available from experimental data. All identification and bounding results in the paper are expressed in terms of $F_{Y_x}(y)=F_{Y|X=x}(y)$ under Assumption \ref{ASEXO2}. However, these results remain valid under conditional exogeneity (by replacing $F_{Y|X=x}(y)$ with $\int F_{Y|X=x,w}(y)\mathfrak{p}(w)dw$) or when experimental data are available (by replacing $F_{Y|X=x}(y)$ with  $F_{Y_x}(y)$).}

We have the following proposition:
\begin{proposition}[Identifiability of RoE]
\label{prop1}
Under SCM ${\cal M}$ and Assumption \ref{ASEXO2},
the expected PO $\mathbb{E}[Y_{x}]$ is identifiable by $\mathbb{E}[Y|X={x}]$.
\end{proposition}

{\citet{Heckman1997} showed that for binary treatment $\Omega_X=\{1,2\}$, the value of the unobserved PO is identified in terms of the observed one under a rank-invariant assumption as follows. }
\begin{assumption}[Rank invariance assumption]\citep{Heckman1997}
\label{POD} 
$F_{Y_1}(y_1)=F_{Y_2}(y_2)$ holds  
for almost every subject whose POs are $(Y_1,Y_2)=(y_1,y_2)$.
\end{assumption}
\begin{proposition}[Identification of the unobserved PO for two POs]\citep{Heckman1997}
\label{theo1}
Under SCM ${\cal M}$ and Assumptions \ref{CON}, \ref{ASEXO2}, and \ref{POD},
given CDF $F_{Y|X=1}$ and $F_{Y|X=2}$,
for almost every subject whose POs are $(Y_1,Y_2)=(y_1,y_2)$,
the value of $y_2$ is identifiable by 
\begin{equation}
y_2=F_{Y|X=2}^{-1}(F_{Y|X=1}(y_1)).
\end{equation}
\end{proposition}
\noindent The function $F_{Y|X=2}^{-1}(F_{Y|X=1}(\cdot))$ is called the shift function in \citep{Doksum1974,Doksum1976}.

\subsubsection{Parameter $\psi$ in \citep{Fay2018}.}
\citet{Fay2018} have studied the causal parameter
\begin{equation} \label{eq-psi}
\psi\defeq \mathbb{P}(Y_2> Y_1),
\end{equation}
which can be considered as a special case of PoR when $K=2$. 
They provide  bounds for $\psi$ {under Assumption \ref{ASEXO2}}, i.e., $\psi_L\leq \psi\leq \psi_U$, where 
\begin{align}
\label{eq7}
&\psi_L=\sup\nolimits_{y}\{F_{Y|X=1}(y)-F_{Y|X=2}(y)\},\\
\label{eq8}
&\psi_U=1-\sup\nolimits_{y}\{F_{Y|X=2}(y)-F_{Y|X=1}(y)\}.
\end{align}
The parameter $\psi$ has been actively studied in causal inference \citep{Fan2010,Huang2016,Lu2019_bound,Gabriel2024}.

\section{Decision-Making Using PoR and PoB}

In this section, we formally introduce new counterfactual decision-making rules inspired by DCM:  PoR and PoB. 

\subsubsection{Probabilities of potential outcome ranking (PoR).}
{Let $\bm{R}=(R_1,R_2,\dots,R_K)$ denote a permutation of elements in $\{1,2,\dots,K\}$, e.g., $\bm{R}=(1,2,\dots,K)$.  }
We denote the set of $K!$ elements that $\bm{R}$ takes as $\mathfrak{R}$.
We then define the PoR metric as follows:
\begin{definition}[PoR]
The probabilities of potential outcome ranking (PoR) of $\bm{r}=(r_1,\dots,r_K)$, for each $\bm{r} \in \mathfrak{R}$, are defined as
\begin{equation}
\label{eq11}
S_{PoR}(\bm{r}) \defeq \mathbb{P}\left(\bm{R}=\bm{r}\right)\defeq\mathbb{P}\left(Y_{r_1}> Y_{r_2}> \dots> Y_{r_K}\right).
\end{equation}
\end{definition}
\noindent The vector $\bm{R}$ can be regarded as a new discrete random variable with domain $\mathfrak{R}$, whose distribution is given by Eq.~\eqref{eq11}.
Since $Y_x$ is a continuous random variable (Assumption \ref{CON}), the probability of POs being tied is zero.
Throughout this paper, we do not consider cases in which POs are tied.

Decision-makers may seek the ranking with the largest PoR, i.e.,
\begin{equation}
\label{eq12}
\bm{r}^{PoR}\defeq\arg\max\nolimits_{\bm{r} \in \mathfrak{R}} S_{PoR}(\bm{r}).
\end{equation}
$\bm{r}^{PoR}$ answers  ({\bf Question 1}).
In concrete terms, decision-makers that employ the PoR metric conclude $r_k^{PoR}$ as the $k$-th preferred action, for $k=1,\dots,K$.
As an example, we present the six patterns ($3!=6$) of ranking for POs when $K=3$ in Table \ref{tab:my_label3}. 
PoR assesses the likelihood of each ranking pattern.


\subsubsection{Probability of achieving the best potential outcome (PoB).}
Some decision-makers may focus on only the action that achieves the largest POs.
For example, in Table \ref{tab:my_label3}, action 1 obtains the largest PO in patterns (A) and (B), that is,  $R_1=1$; $R_1=2$ in patterns (C) and (D); and $R_1=3$ in patterns (E) and (F). 
We define the probability of an action achieving the best potential outcome (PoB) as the marginalized form of PoR.
\begin{definition}[PoB]
We define 
the probability of action $r_1$ achieving the best potential outcome (PoB), for each $r_1 \in \{1,2,\dots,K\}$, as
\begin{equation}
\label{eq14}
S_{PoB}(r_1) \defeq \mathbb{P}\left(R_1=r_1\right)=\sum\nolimits_{\bm{r} \in \mathfrak{R};R_1=r_1} \mathbb{P}(\bm{R}=\bm{r}).
\end{equation}
\end{definition}
\noindent Alternatively, 
\begin{align}
S_{PoB}(r_1)=\mathbb{P}\left(\arg\max\nolimits_{k=1,\dots,K}\{Y_k\}=r_1\right).
\end{align}
Decision-makers seek the action that has the largest PoB: 
\begin{equation}
r_1^{PoB}\defeq\arg\max\nolimits_{r_1=1,2,\dots,K} S_{PoB}(r_1).
\end{equation}
$r_1^{PoB}$ answers  ({\bf Question 2}).
Decision-makers employing PoB conclude $r_1^{PoB}$ as the most preferred action.
Note that $r_1^{PoB}$ may not be equal to $r_1^{PoR}$, which is the most preferred action by $\bm{r}^{PoR}$.

\subsubsection{Example 1 (Homogeneous causal effects).}
We consider SCM $Y=X+U^Y$ with $\mathbb{E}[U^Y]=0$ and $K=3$. We have $Y_k=k+U^Y$, and the individual causal effect (ICE) $Y_3-Y_2 = Y_2-Y_1=1$ for every subject, meaning that
the causal effect is homogeneous. 
The expectations of $Y_1$, $Y_2$, and $Y_3$ are $1$, $2$, and $3$, respectively.
$Y_3> Y_2> Y_1$ holds for every subject. Then $\mathbb{P}(\bm{R}=(3,2,1))=1$ and PoR of all other ranking patterns is equal to 0.
Therefore, $\bm{r}^{PoR}=(3,2,1)$. Finally, $S_{PoB}(3) = \mathbb{P}(R_1=3)=1$ with PoB of all other actions being 0; therefore $r_1^{PoB}=3$. We conclude that action 3 is the optimal choice from  the perspectives of all three metrics RoE, PoR, and PoB. 
{In general, the three metrics lead to the same optimal choice 
when the causal effect is  homogeneous, meaning $Y_x - Y_{x-1}$ is constant.}

\subsubsection{Example 2 (Heterogeneous causal effects).}
We consider SCM $Y=XU^Y$ with $K=3$, $\mathbb{E}[U^Y]=0$, $\mathbb{P}(U^Y=0)=0$, and $\mathbb{P}(U^Y>0)=0.6$.  We have $Y_k=kU^Y$, and ICE $Y_3-Y_2 = Y_2-Y_1=U^Y$, which varies between subjects.
The expectations of $Y_1$, $Y_2$, and $Y_3$ are all $0$.
We have $\mathbb{P}(\bm{R}=\bm{r})=\mathbb{P}(Y_{r_1}> Y_{r_2}> Y_{r_3})=\mathbb{P}({r_1}U^Y> {r_2}U^Y> {r_3}U^Y)$, 
 $\mathbb{P}({r_1}U^Y> {r_2}U^Y> {r_3}U^Y|U^Y> 0)=\mathbb{I}({r_1}> {r_2}> {r_3})$, and $\mathbb{P}({r_1}U^Y> {r_2}U^Y> {r_3}U^Y|U^Y<0)=\mathbb{I}({r_3}> {r_2}> {r_1})$.
Then, $\mathbb{P}(\bm{R}=\bm{r})=\mathbb{I}({r_1}> {r_2}> {r_3})\mathbb{P}(U^Y> 0)+\mathbb{I}({r_3}> {r_2}> {r_1})\mathbb{P}(U^Y<0)$.
We obtain $\mathbb{P}(\bm{R}=(3,2,1))=0.6$,  $\mathbb{P}(\bm{R}=(1,2,3))=0.4$, and the PoR of all other ranking patterns are equal to 0. 
Therefore, $\bm{r}^{PoR}=(3,2,1)$. 
Finally, $S_{PoB}(3)=0.6$, $S_{PoB}(1)=0.4$, and $S_{PoB}(2)=0$. 
We conclude that the three actions are tied from the perspective of RoE; action 3 is the optimal choice, followed by action 2, with action 1 ranked last from the perspective of PoR; and action 3 is the optimal choice, followed by action 1, with action 2 ranked last from the perspective of PoB. 


\begin{table}[tb]
    \centering
    \begin{tabular}{c|ccc}
    \hline
     Ranking Patterns & $R_1$ & $R_2$&$R_3$ \\
       \hline\hline
      (A) & 1 & 2 &3 \\
      (B) & 1 & 3 &2 \\
       (C) &2 & 1 &3 \\
       (D) &2 & 3 &1 \\
      (E) & 3 & 1 &2 \\
      (F) & 3 & 2 &1 \\
\hline
\end{tabular}
    \caption{Six patterns of rankings for POs $(Y_1,Y_2,Y_3)$.}
    \label{tab:my_label3}
\end{table}

\textbf{Remark (RoE, PoR, PoB).}
{While PoR is constructed to assess the preference rankings of actions, one can interpret $r_1^{PoR}$, the first element in $\bm{r}^{PoR}$, as the optimal action. On the other hand, although RoE and PoB are designed to select a single best action, they can be interpreted to measure the quality of a ranking $\bm{r}=(r_1,\dots,r_K)$ of actions by the following metrics: 
\begin{align}
    S_{RoE}(\bm{r}) = \mathbb{I}\left(\mathbb{E}[Y_{r_1}] \ge \mathbb{E}[Y_{r_2}] \ge \dots \ge \mathbb{E}[Y_{r_K}]\right),
\end{align}
and 
\begin{align}
    S_{PoB}(\bm{r}) = \mathbb{I}\left(S_{PoB}(r_1) \ge S_{PoB}(r_2) \ge \dots \ge S_{PoB}(r_K)\right),
\end{align}
respectively. In other words, RoE and PoB select the best ranking of actions as $\bm{r}^{RoE}\defeq\arg\max\nolimits_{\bm{r} \in \mathfrak{R}} S_{RoE}(\bm{r})$ and $\bm{r}^{PoB}\defeq\arg\max\nolimits_{\bm{r} \in \mathfrak{R}} S_{PoB}(\bm{r})$, respectively.
}

\textbf{Remark (PoC).}
The probabilities of causation (PoC) are a family of counterfactual probabilities quantifying whether one event was the real cause of another in a given scenario \citep{Pearl1999,Tian2000}.
Personalized decision-making based on PoC with discrete treatments and outcomes has been explored by \citet{Mueller2023} and \citet{ALi2024}.
PoC with continuous outcomes were examined by \citet{Kawakami2024}. 
PoC expressions like $\mathbb{P}(Y_{x_0} < y_1 \leq Y_{x_1} \cdots < y_p \leq Y_{x_p})$ in \cite{Kawakami2024} may appear similar to PoR. 
The key difference is that PoR does not have thresholds like $y_i$ for POs. 
For binary $Y$ and $X$ ($K = 2$) taking values in $\{1, 2\}$, our PoR and PoB both reduce to the probability of necessity and sufficiency (PNS) \citep{Pearl1999,Tian2000}. Specifically,  $S_{PoR}(2, 1) = S_{PoB}(2)=\mathbb{P}(Y_2>Y_1)=\mathbb{P}(Y_2=2,Y_1=1)$.


\section{Identification and Estimation of PoR and PoB}

Next, we investigate the identification challenges associated with PoR and PoB. 
By Proposition \ref{prop1}, the RoE metric is identifiable under the exogeneity assumption.
However, the identification of PoR and PoB is more challenging and requires additional assumptions. 

\textbf{Identification of the values of the unobserved POs.} 
We extend the identification result by \citet{Heckman1997} given in Proposition~\ref{theo1} to the setting of multiple actions.
\begin{assumption}[Rank invariance assumption for multiple actions]
\label{PODm} 
$F_{Y_1}(y_1)=F_{Y_2}(y_2)=\dots=F_{Y_K}(y_K)$ holds  
for almost every subject whose POs are $(Y_1,Y_2,\dots,Y_K)=(y_1,y_2,\dots,y_k)$.
\end{assumption}
\noindent This assumption means that a subject whose PO under action 1 lies in the top $q$ percent also has POs in the top $q$ percent under all other actions $k=2,\dots,K$.
Then, given $Y_1$, the values of other unobserved POs are identifiable as follows:
\begin{lemma}[Identification of the unobserved POs]
\label{theo2}
Under SCM ${\cal M}$ and Assumptions \ref{CON}, \ref{ASEXO2}, and \ref{PODm},
given CDF $F_{Y|X=k}$ for $k=1,\dots,K$,
for almost every subject whose POs are $(Y_1,\dots,Y_K)=(y_1,\dots,y_K)$,
the values of $y_2, \dots, y_K$ are identifiable by 
\begin{equation}
y_k = F_{Y|X=k}^{-1}(F_{Y|X=1}(y_1)).
\end{equation}
\end{lemma}

\subsubsection{Identification of PoR and PoB.}
We apply the identification result in Lemma~\ref{theo2} to the identification of PoR and PoB.
\begin{theorem}[Identification of PoR]
\label{theorem2}
Under SCM ${\cal M}$ and Assumptions \ref{CON}, \ref{ASEXO2}, and \ref{PODm}, $S_{PoR}(\bm{r})$ 
is identified by $\sigma(\bm{r})$ for any $\bm{r} \in \mathfrak{R}$, given by
\begin{align}
\sigma(\bm{r})
&=\mathbb{P}\Big(Y> F_{Y|X={r_2}}^{-1}(F_{Y|X={r_1}}(Y))> \dots\nonumber\\
&\hspace{1cm}> F_{Y|X={r_K}}^{-1}(F_{Y|X={r_1}}(Y))\Big|X={r_1}\Big).
\end{align}
\end{theorem}
\begin{theorem}[Identification of PoB]
\label{identi_PoB}
Under SCM ${\cal M}$ and Assumptions \ref{CON}, \ref{ASEXO2}, and \ref{PODm},
$S_{PoB}(r_1)$ is identified by $\eta(r_1)$ for any $r_1 \in \{1,\dots,K\}$, given by
\begin{align}\label{eq-pobe2}
\eta(r_1)&=\mathbb{P}\Big(Y> F_{Y|X={r_2}}^{-1}(F_{Y|X={r_1}}(Y)),\nonumber\\
&\hspace{0.5cm}Y> F_{Y|X={r_3}}^{-1}(F_{Y|X={r_1}}(Y)),\dots,\nonumber\\
&\hspace{0.5cm}Y > F_{Y|X={r_K}}^{-1}(F_{Y|X={r_1}}(Y))\Big|X={r_1}\Big),
\end{align}
where $(r_2,\dots,r_K)$ is an arbitrary permutation of elements in $\{1,2,\dots,K\}\setminus\{r_1\}$. 
\end{theorem}
\noindent 
When $K = 2$, Theorem~\ref{theorem2} establishes the identification of the causal parameter $\psi$ in (\ref{eq-psi}), which was not addressed in \cite{Fay2018}.


\textbf{Estimation of PoR and PoB} 
We present an estimation method for PoR and PoB following the approach proposed in \citep{Doksum1974, Doksum1976}.
We denote the samples {of $Y$ corresponding to $X=k$} (or samples from the intervention $do(X=k)$ in the setting of given experimental data) as
\begin{equation}
{\cal D}^{(k)}=\{Y^{(k)}_1,\dots,Y^{(k)}_{N^{(k)}}\}
\end{equation} 
for $k=1,2,\dots,K$, where 
$N^{(k)}$ is the sample size of ${\cal D}^{(k)}$.
We sort dataset ${\cal D}^{(k)}$ and denote 
\begin{equation}
{\cal D}^{(k)}=\{Y^{(k)}_{(1)},\dots,Y^{(k)}_{(N^{(k)})}\}
\end{equation} 
where $Y^{(k)}_{(1)}\leq\dots\leq Y^{(k)}_{(N^{(k)})}$.
Then, our estimates of $\sigma(\bm{r})$ are given by
\begin{align}
&\hat{\sigma}(\bm{r})=\frac{1}{N^{(r_1)}}\sum_{i=1}^{N^{(r_1)}}\mathbb{I}\Bigg(Y_i^{(r_1)}> Y^{(r_2)}_{\left(\left\lfloor N^{(r_2)}\hat{F}_{Y|X=r_2}(Y_i^{(r_1)}) \right\rfloor+1\right)}\nonumber\\
&\hspace{1.5cm}> \dots> Y^{(r_K)}_{\left(\left\lfloor N^{(r_K)}\hat{F}_{Y|X=r_K}(Y_i^{(r_1)}) \right\rfloor+1\right)}\Bigg),
\end{align}
where  
\begin{align}\label{eq-cdf}
\hat{F}_{Y|X=r_k}(y)=\frac{1}{N^{(r_k)}}\sum\nolimits_{j=1}^{N^{(r_k)}} \mathbb{I}(Y^{(r_k)}_{j}\leq y).
\end{align}
Our estimates of $\eta(r_1)$ are given by 
\begin{align}
\hat{\eta}(r_1)&=\frac{1}{N^{(r_1)}}\sum_{i=1}^{N^{(r_1)}}\mathbb{I}\Bigg(Y_i^{(r_1)}> Y^{(r_2)}_{\left(\left\lfloor N^{(r_2)}\hat{F}_{Y|X=r_2}(Y_i^{(r_1)}) \right\rfloor+1\right)},\nonumber\\
&\hspace{0.5cm} Y_i^{(r_1)}> Y^{(r_3)}_{\left(\left\lfloor N^{(r_3)}\hat{F}_{Y|X=r_3}(Y_i^{(r_1)}) \right\rfloor+1\right)},\dots,\nonumber\\
&\hspace{0.5cm}Y_i^{(r_1)}> Y^{(r_K)}_{\left(\left\lfloor N^{(r_K)}\hat{F}_{Y|X=r_K}(Y_i^{(r_1)}) \right\rfloor+1\right)}\Bigg).
\end{align} 
The computational complexities for $\hat{\sigma}(\bm{r})$ and $\hat{\eta}(r_1)$ are ${\cal O}\left(\sum_{k=1}^K N^{(k)} \log N^{(k)}\right)$, as each sorting operation on ${\cal D}^{(k)}$ requires ${\cal O}(N^{(k)} \log N^{(k)})$ computations. They are consistent estimators as follows:
\begin{theorem}
\label{theorem3}
Under SCM ${\cal M}$ and Assumptions \ref{CON}, \ref{ASEXO2}, and \ref{PODm}, 
$\hat{\sigma}(\bm{r})$ and $\hat{\eta}(r_1)$ converge in probability to 
$S_{PoR}(\bm{r})$ and $S_{PoB}(r_1)$, respectively, at the rate ${\cal O}_p(\sum_{k=1}^K{{N^{(k)}}^{-1/2}})$.
\end{theorem} 
\noindent
When $N^{(k)}=N/K$ for any $k$, the rate becomes ${\cal O}_p(K^{3/2}{N^{-1/2}})$, which worsens as  $K$ increases. 

\section{Bounding PoR and PoB}
The exogeneity (Assumption \ref{ASEXO2}) is considered reasonable for experimental data by many researchers.
In contrast, some researchers consider the rank invariance assumption (Assumption \ref{PODm})  overly restrictive; therefore, we derive bounds of PoR and PoB without relying on Assumption \ref{PODm}. 

Since $S_{PoR}(\bm{r})$ is given as
\begin{equation}
\begin{aligned}
&\mathbb{P}\left((Y_{r_1}> Y_{r_2})\land(Y_{r_2}>Y_{r_3})\land \dots\land(Y_{r_{K-1}}> Y_{r_K})\right),
\end{aligned}    
\end{equation}
$S_{PoR}(\bm{r})$ is bounded as follows from the Fr\'{e}chet inequalities \citep{Frechet1935,Frechet1960}:
\begin{equation}
\begin{aligned}
&\max\left\{\sum\nolimits_{k=1}^{K-1}\mathbb{P}(Y_{r_{k}}> Y_{r_{k+1}})-K+2,0\right\}\leq S_{PoR}(\bm{r})\\
&\leq \min\nolimits_{k=1,\dots,K-1}\left\{\mathbb{P}(Y_{r_{k}}> Y_{r_{k+1}})\right\}.
\end{aligned}    
\end{equation}
Then, we have the following results using the bounds of $\psi$ given in Eqs.~\eqref{eq7} and \eqref{eq8}:  
\begin{theorem}[Bounds of PoR]
\label{theoremb_1}
Under SCM ${\cal M}$ and Assumptions \ref{CON} and \ref{ASEXO2},
for any $\bm{r} \in \mathfrak{R}$, 
$S_{PoR}(\bm{r})$ is bounded by $\sigma_L(\bm{r})\leq S_{PoR}(\bm{r}) \leq \sigma_U(\bm{r})$, where
\begin{align}
\label{eq31d}
&\sigma_L(\bm{r})=\max\Big\{\sum\nolimits_{k=1}^{K-1}\sup_y\{F_{Y|X=r_{k+1}}(y)\\
&\hspace{2cm}-F_{Y|X=r_{k}}(y)\}-K+2,0\Big\},\nonumber\\
\label{eq32d}
&\sigma_U(\bm{r})= \min\nolimits_{k=1,\dots,K-1}\Big\{1-\sup_y\{F_{Y|X=r_k}(y)\\
&\hspace{3cm}-F_{Y|X=r_{k+1}}(y)\}\big\}.\nonumber
\end{align}
\end{theorem}
\noindent The bounds in Theorem~\ref{theoremb_1} for $K=2$ reduce to the bounds of $\psi$. 
Similarly, since $S_{PoB}(r_1)$ is given as
\begin{align}
&\mathbb{P}\left((Y_{r_1}> Y_{r_2})\land(Y_{r_1}>Y_{r_3})\land \dots\land(Y_{r_1}> Y_{r_K})\right),
\end{align}
$S_{PoB}(r_1)$ is bounded as follows from the Fr\'{e}chet inequalities:
\begin{equation}
\begin{aligned}
&\max\left\{\sum\nolimits_{k=1}^{K-1}\mathbb{P}(Y_{r_{1}}> Y_{r_{k+1}})-K+2,0\right\}\leq S_{PoB}(r_1)\\
&\leq \min\nolimits_{k=1,\dots,K-1}\left\{\mathbb{P}(Y_{r_{1}}> Y_{r_{k+1}})\right\}.
\end{aligned}    
\end{equation}
Then, we have the following results: 
\begin{theorem}[Bounds of PoB]
\label{theoremb_2}
Under SCM ${\cal M}$ and Assumptions \ref{CON} and \ref{ASEXO2},
for any $r_1=1,2,\dots,K$,
$S_{PoB}(r_1)$ is bounded by $\eta_L(r_1)\leq S_{PoB}(r_1) \leq \eta_U(r_1)$, where
\begin{align}
\label{eq33d}
&\eta_L(r_1)=\max\Big\{\sum\nolimits_{k=1}^{K-1}\sup_y\{F_{Y|X=r_{k+1}}(y)\\
&\hspace{2cm}-F_{Y|X=r_{1}}(y)\}-K+2,0\Big\},\nonumber\\
\label{eq34d}
&\eta_U(r_1)= \min\nolimits_{k=1,\dots,K-1}\Big\{1-\sup_y\{F_{Y|X=r_1}(y)\\
&\hspace{3cm}-F_{Y|X=r_{k+1}}(y)\}\big\}.\nonumber
\end{align}
\end{theorem}
\noindent These bounds of PoR and PoB are not sharp except for the case of $K=2$.

\subsubsection{Estimation of Bounds.}
We can estimate the bounds in Eqs.~\eqref{eq31d}, \eqref{eq32d}, \eqref{eq33d}, and \eqref{eq34d} by plugging the estimates of conditional CDF $\hat{F}_{Y|X=r_k}(y)$ given in \eqref{eq-cdf}. 
The bound estimators are 
\begin{align}
&\hat{\sigma}_L(\bm{r})=\max\Big\{\sum\nolimits_{k=1}^{K-1}\hat{\sup}_y\{\hat{F}_{Y|X=r_{k+1}}(y)\\
&\hspace{3cm}-\hat{F}_{Y|X=r_{k}}(y)\}-K+2,0\Big\},\nonumber\\
&\hat{\sigma}_U(\bm{r})= \min\nolimits_{k=1,\dots,K-1}\Big\{1-\hat{\sup}_y\{\hat{F}_{Y|X=r_k}(y)\\
&\hspace{3cm}-\hat{F}_{Y|X=r_{k+1}}(y)\}\big\},\nonumber\\
&\hat{\eta}_L(r_1)=\max\Big\{\sum\nolimits_{k=1}^{K-1}\hat{\sup}_y\{\hat{F}_{Y|X=r_{k+1}}(y)\\
&\hspace{3cm}-\hat{F}_{Y|X=r_{1}}(y)\}-K+2,0\Big\},\nonumber\\
&\hat{\eta}_U(r_1)= \min\nolimits_{k=1,\dots,K-1}\Big\{1-\hat{\sup}_y\{\hat{F}_{Y|X=r_1}(y)\\
&\hspace{3cm}-\hat{F}_{Y|X=r_{k+1}}(y)\}\big\},\nonumber
\end{align}
where $\hat{\sup}_y$ denotes the supremum operator evaluated over a finite grid. Specifically, for $\Omega_Y=[a,b]$,   
$\hat{\sup}_{y \in \Omega_Y} f(y) := \max_{i=1,\dots,M} f(y^i)$,
where $y^i=a+\frac{i-1}{M-1}(b-a)$ for $i=1,\dots,M$.
The computational complexities for $\hat{\sigma}_L(\bm{r})$, $\hat{\sigma}_U(\bm{r})$, $\hat{\eta}_L(r_1)$, and $\hat{\eta}_U(r_1)$ are ${\cal O}(KM\sum_{k=1}^K N^{(k)})$.
These estimators are consistent as follows:
\begin{theorem}
\label{theorem4}
Under SCM ${\cal M}$ and Assumptions \ref{CON} and \ref{ASEXO2}, 
if $\Omega_Y$ is bounded,
$\hat{\sigma}_L(\bm{r})$, $\hat{\sigma}_U(\bm{r})$, $\hat{\eta}_L(r_1)$, and $\hat{\eta}_U(r_1)$ converge in probability to $\sigma_L(\bm{r})$, $\sigma_U(\bm{r})$, $\eta_L(r_1)$, and $\eta_U(r_1)$, respectively, at the rate ${\cal O}_p(1/M+\sum_{k=1}^K{{N^{(k)}}^{-1/2}})$.
\end{theorem}

\section{Numerical Experiments}

\begin{table*}[t]
    \centering
    \scalebox{0.9}{
    \begin{tabular}{cc|cccc}
    \hline
      Setting&  Measures & $N'=30$& $N'=300$& $N'=3000$& True Value \\
        \hline\hline
       (A) & $S_{PoR}(1, 2, 3)$  & {0.557 [0.149,0.918]}   & 0.580 [0.421,0.720] & 0.635 [0.560,0.638] & 0.666\\
       (A) & $S_{PoB}(1)$    & 0.561 [0.133,0.884] & 0.645 [0.439,0.763] &  0.650 [0.580,0.702]  & 0.666 \\
       \hline
       (B) & UB  of $S_{PoR}(1, 2, 3)$ & 0.813 [0.683,0.900] & 0.857 [0.806,0.890] &  0.873 [0.855,0.887]  & $0.500$ \\
       (B) & LB  of $S_{PoR}(1, 2, 3)$ & 0.000 [0.000,0.000] & 0.000 [0.000,0.000] &  0.000 [0.000,0.000]  & $0.500$ \\
       (B) & UB  of $S_{PoB}(1)$ & 0.767 [0.649,0.867] & 0.763 [0.708,0.813] &  0.767 [0.745,0.798]  & $0.500$ \\
       (B) & LB  of $S_{PoB}(1)$ & 0.000 [0.000,0.000] & 0.000 [0.000,0.000] &  0.000 [0.000,0.000]  & $0.500$ \\
       \hline
    \end{tabular}
    }
    \caption{Results of experiments (three candidate actions): Mean estimates with 95\% CI. UB: upper bound; LB: lower bound. 
    }
    \label{tab:my_label4}
\end{table*}

\begin{table*}[tb]
    \centering
    \scalebox{0.9}{
    \begin{tabular}{cc|cccc}
    \hline
     Setting&   Measures & $K=3$ & $K=5$& $K=10$& $K=20$ \\
        \hline\hline
    (C) &  
    $S_{PoR}(1,2,\dots,K)$ & 0.022 [0.001,0.067] & 0.047 [0.007,0.104]  & 0.117 [0.049,0.210] & 0.256 [0.147,0.391]  \\
    (C) &  
    $S_{PoB}(1)$ & 0.015 [0.001,0.043] & 0.015 [0.001,0.044] & 0.016 [0.004,0.045] &  0.023 [0.012,0.048]    \\
      \hline
    \end{tabular}
    }
    \caption{Results of experiments (many candidate actions): Absolute errors of the estimates from the truth, with 95\% CIs.}
    \label{tab:my_label5}
\end{table*}

We illustrate the finite-sample properties of the estimators.
To our knowledge, there are no existing baselines for estimating PoR or PoB when the number of actions $K \geq 3$.

\subsection{Three Candidate Actions}

First, we evaluate the performance of the PoR and PoB estimators  in the case of three candidate actions $K=3$.

\textbf{Settings.}
We assume the following SCMs:
\begin{align}
&\text{(A): } Y:=-(X-4)U^Y, U^Y \sim \text{Unif}(-0.5,1),\\
&\text{(B): } Y:=-XU^Y+{U'}^Y,\nonumber\\
&\hspace{1cm}U^Y \sim {\cal N}(-1,1),  {U'}^Y \sim \text{Unif}(-1,1),
\end{align}
where $X$ takes $\{1,2,3\}$ and $\mathbb{P}(X=k)=1/3$ for any $k=1,2,3$.
Both SCMs satisfy Assumptions \ref{CON} and \ref{ASEXO2}.
SCM (A) satisfies the rank invariance assumption (Assumption \ref{PODm}), whereas SCM (B) violates it. We use SCM (A) to illustrate the performance of our PoR and PoB estimators, and use SCM (B) to illustrate the performance of the bounds estimators.
We perform 100 simulation runs and report the mean along with the 95\% confidence intervals (CI) for each statistic. We use sample size $N=KN^{(k)}$ with varying $N^{(k)}=N'=30, 300, 3000$ for $k=1,2,3$.

\textbf{Results.}
We present in Table \ref{tab:my_label4} the estimates for $S_{PoR}(1, 2, 3)$ and $S_{PoB}(1)$ 
under SCM (A), and the upper bound (UB) and lower bound (LB) for $S_{PoR}(1, 2, 3)$ and $S_{PoB}(1)$ 
under SCM (B). 
The results under SCM (A) show that the estimates become more reliable as the sample size increases. 
Our model correctly predicts that  the most probable ranking is $(1, 2, 3)$, and  the  action that is the most likely to yield the best outcome is $1$ (both probabilities are greater than 0.5). 
The results under SCM (B) show that the ground truth value lies within the computed bounds, and the CIs of estimates become narrower as the sample size increases. 
All the computed lower bounds are non-informative for this model.


\subsection{Many Candidate Actions}

Next, we examine the reliability of the estimated PoR and PoB under the rank invariance assumption
as the number of actions $K$ increases.

\textbf{Setting.}
We assume the following SCM:
\begin{align}
&\text{(C): } Y:=-XU^Y, U^Y \sim \text{Unif}(-1,1),
\end{align}
where $X$ takes $\{1,2,\dots,K\}$ and $\mathbb{P}(X=k)=1/K$ for any $k=1,2,\dots,K$.
This SCM satisfies Assumptions \ref{CON}, \ref{ASEXO2}, and \ref{PODm}.
We perform 100 simulation runs and report the mean along with the 95\% CI for each estimate, using a fixed per action sample size of $N' = 3000$ and varying the number of actions $K \in \{3, 5, 10, 20\}$.

\textbf{Results.}
Table \ref{tab:my_label5} shows the absolute errors of 
estimated $P_{PoR}(1,2,\dots,K)$ and $S_{PoB}(1)$ from their ground truths.
As the number of candidate actions $K$ increases, the estimates deviate further from the ground truth. 
{We present curves of absolute errors of estimates vs. $K$ in Figures~\ref{fig:PoRK} and~\ref{fig:PoBK} in Appendix B. The errors roughly increase linearly in terms of $K$, matching with the convergence rate results in Theorem~\ref{theorem3}.} 

\section{Application to Real-World Data}


\subsubsection{Dataset.}
We pick up an open medical  dataset ``coagulation" \citep{Kropf2000} in (\texttt{https://cran.r-project.org/web/packages/\\SimComp/SimComp.pdf}).
Three extracorporeal circulation systems in heart-lung machines were evaluated: treatments ``H" and ``B", alongside the standard ``S". 
The study included 12 male adult patients each for treatments S and H, and 11 for treatment B. 
The analysis focused on laboratory parameters related to blood coagulation, using three primary endpoints expressed as the ratio of post- to pre-surgery values. 
Higher values indicate a more beneficial or favorable treatment effect.
\citet{Kropf2000} compared the mean vectors of outcomes using their developed statistical hypothesis testing method.
We focus on an endpoint referred to as "\texttt{Thromb.count}" among the three endpoints.
Decision-makers, such as doctors, try to choose the most effective 
treatment to achieve better endpoint results.
We denote POs as $Y_{S}$, $Y_{H}$, and $Y_{B}$.

\textbf{Analysis.}
We adopt Assumption \ref{ASEXO2} (exogeneity), as \citet{Kropf2000} did not report the presence of any confounding variables.
We impose Assumption \ref{CON}, as the outcome variable is continuously distributed.
We assume the rank invariant assumption (Assumption \ref{PODm}).
The rank invariance assumption implies that a patient who belongs to the top $q$ percent in terms of the endpoint under treatment $S$ also belongs to the top $q$ percent under treatments $B$ and $H$.
We conduct the bootstrapping \citep{Efron1979} to provide the means and 95\% confidence intervals (CIs) for each estimator.

\textbf{Results.}
We show the estimates of RoE, six patterns of PoR, and three PoB in Table \ref{tab:my_label6}.
We also compute the bounds for each estimator and present the results in Appendix C. 
The bounds are relatively wide. 
{Overall, due to the small sample size, the CIs of the estimates for different treatments have considerable overlap, and it's difficult to draw reliable conclusions about the effectiveness ranking of the three treatments. }

{In the following, we analyze the results purely based on the mean values. From Table \ref{tab:my_label6}, all three metrics conclude that treatment $B$ is the most effective. After $B$, RoE and PoB will (slightly) prefer treatment $H$ over $S$. However, the most probable ranking by PoR is $(B,S,H)$, which prefers $S$ over $H$. 
Our proposed PoR and PoB metrics constitute alternative frameworks for decision-making, which may  yield action choices that in some cases differ significantly from those recommended by traditional RoE. Furthermore, they provide more nuanced insights into treatment effects. For example,  PoR reveals  that, for 8.1\% of patients,  treatment $S$ is the most effective, followed by $H$, and then $B$. This indicates the existence of patients whose outcome ranking is completely opposite to that implied by RoE. As another example, although treatment $B$ is the most effective on average, PoB reveals that $B$ is the optimal choice for only 49\% of patients, not for the other 51\%. The information provided by PoR and PoB enables clinicians to better understand the benefits and risks associated with each treatment choice.
}

\begin{table}[t]
    \centering
    \scalebox{0.9}{
    \begin{tabular}{c|c}
    \hline
    RoE, PoR, or PoB &  Estimates\\
        \hline\hline
    $\mathbb{E}[Y_B]$  & \textbf{0.994} [0.859,1.128] \\
      $\mathbb{E}[Y_H]$    & 0.917 [0.757,1.087] \\
      $\mathbb{E}[Y_S]$  & 0.873 [0.782,0.984] \\
      \hline
      $S_{PoR}(B,H,S)$  & 0.215 [0.000,0.500] \\
      $S_{PoR}(B,S,H)$  & \textbf{0.277} [0.083,0.583] \\
      $S_{PoR}(H,B,S)$  & 0.172 [0.000,0.467] \\
      $S_{PoR}(H,S,B)$  & 0.123 [0.000,0.333] \\
      $S_{PoR}(S,B,H)$  & 0.132 [0.000,0.333] \\
      $S_{PoR}(S,H,B)$  & 0.081 [0.000,0.250] \\
     \hline
      $S_{PoB}(B)$  & \textbf{0.490} [0.182,0.818]\\
      $S_{PoB}(H)$  & 0.295 [0.083,0.583]\\
      $S_{PoB}(S)$  & 0.216 [0.000,0.500]\\
      \hline
    \end{tabular}
    }
    \caption{Mean estimates with 95\% CIs of the RoE, PoR, and PoB metrics.
    We highlight the highest values. 
    }
    \label{tab:my_label6}
\end{table}

\section{Conclusion}

We propose two novel counterfactual decision-making rules. 
Decision-making based on PoR or PoB differs from that based on the traditional RoE.
We do not claim that one approach is universally better than the other, as the choice often depends on the decision-maker’s perspective and preferences.
An individual who prioritizes the expected outcome—possibly reflecting a more aggressive decision-making style—may prefer actions based on RoE, even when the probability of receiving the best individual outcome is low.
In contrast, a more conservative decision-maker may favor PoB, prioritizing the likelihood of achieving the best individual outcome, even if it means foregoing the action with the highest expected outcome.
There may also be decision-makers who aim to balance different perspectives, incorporating considerations of all three RoE, PoR, and PoB into their choices.

Evaluating the PoR and PoB metrics poses greater challenges than evaluating RoE. Our identification results rely on the restrictive rank invariance assumption, which may limit their practical applications. We provide bounds that do not rely on this assumption, but they are not sharp in general.  Deriving tight bounds for  PoR and PoB  remains an important direction for future research. We see personalized decision-making using PoB or PoR as a promising direction for future research, offering an alternative to the traditional RoE, such as in  multi-armed bandit (MAB) or reinforcement learning problems.

\section*{Acknowledgements}

The authors thank the anonymous reviewers for their time
and thoughtful comments.

\bibliography{aaai25}

\newpage
\onecolumn
\appendix

\section*{\LARGE Appendix to ``Potential Outcome Rankings for Counterfactual Decision Making"}

\section*{Appendix A: Proof}
\label{appA}

In this appendix, we provide proofs of lemmas and theorems in the body of the paper.\\

\noindent{\bf Lemma \ref{theo2}}
{\it
Under SCM ${\cal M}$ and Assumptions \ref{CON}, \ref{ASEXO2}, and \ref{PODm},
given CDF $F_{Y|X=k}$ for $k=1,\dots,K$,
for almost every subject whose POs are $(Y_1,\dots,Y_K)=(y_1,\dots,y_K)$,
the values of $y_2, \dots, y_K$ are identifiable by 
\begin{equation}
y_k = F_{Y|X=k}^{-1}(F_{Y|X=1}(y_1)).
\end{equation}
}

\begin{proof}
From Assumption \ref{PODm}, we have
\begin{equation}
F_{Y_1}(y_1)=F_{Y_2}(y_2)=\dots=F_{Y_K}(y_K).
\end{equation}
From Assumption \ref{CON}, we have
\begin{equation}
y_k = F_{Y_k}^{-1}(F_{Y_1}(y_1)).
\end{equation}
From Assumption \ref{ASEXO2}, 
we have 
\begin{align}
y_k&=F_{Y_k}^{-1}(F_{Y_1}(y_1))\\
&=F_{Y_k|X=k}^{-1}(F_{Y_1|X=1}(y_1))\ \ (\because Y_k \indep X, \forall k)\\
&=F_{Y|X=k}^{-1}(F_{Y|X=1}(y_1))\ \ (\because X=k \Rightarrow Y_k=Y, \forall k).
\end{align}
Then, we have Lemma \ref{theo2}.
\end{proof}

\noindent{\bf Theorem \ref{theorem2}.}
{\it
Under SCM ${\cal M}$ and Assumptions \ref{CON}, \ref{ASEXO2}, and \ref{PODm}, $S_{PoR}(\bm{r})$ 
is identified by $\sigma(\bm{r})$ for any $\bm{r} \in \mathfrak{R}$, given by
\begin{align}
\sigma(\bm{r})
&=\mathbb{P}\Big(Y> F_{Y|X={r_2}}^{-1}(F_{Y|X={r_1}}(Y))> \dots > F_{Y|X={r_K}}^{-1}(F_{Y|X={r_1}}(Y))\Big|X={r_1}\Big).
\end{align}
}

\begin{proof}
From Lemma \ref{theo2}, we have $Y_{r_k}=F_{Y|X={r_k}}^{-1}(F_{Y|X={r_1}}(Y_{r_1}))$ since $y_k = F_{Y|X=k}^{-1}(F_{Y|X=1}(y_1))$ holds for any subjects.
Then, we have
\begin{align}
&\mathbb{P}\left(Y_{r_1}> Y_{r_2}> \dots> Y_{r_K}\right)\\
&=\mathbb{P}(Y_{r_1}> F_{Y|X={r_2}}^{-1}(F_{Y|X={r_1}}(Y_{r_1}))> \dots> F_{Y|X={r_K}}^{-1}(F_{Y|X={r_1}}(Y_{r_1})))\\
&=\mathbb{P}(Y> F_{Y|X={r_2}}^{-1}(F_{Y|X={r_1}}(Y))> \dots> F_{Y|X={r_K}}^{-1}(F_{Y|X={r_1}}(Y))|X={r_1}),
\end{align}
{where the last step is by Assumption \ref{ASEXO2}.} Then, we have Theorem \ref{theorem2}.
\end{proof}

\noindent{\bf Theorem \ref{identi_PoB}.}
{\it
Under SCM ${\cal M}$ and Assumptions \ref{CON}, \ref{ASEXO2}, and \ref{PODm},
$S_{PoB}(r_1)$ is identified by $\eta(r_1)$ for any $r_1 \in \{1,\dots,K\}$, given by
\begin{align}
\eta(r_1)&=\mathbb{P}\Big(Y> F_{Y|X={r_2}}^{-1}(F_{Y|X={r_1}}(Y)),Y> F_{Y|X={r_3}}^{-1}(F_{Y|X={r_1}}(Y)),\dots,Y > F_{Y|X={r_K}}^{-1}(F_{Y|X={r_1}}(Y))\Big|X={r_1}\Big),
\end{align}
where $(r_2,\dots,r_K)$ is an arbitrary permutation of elements in $\{1,2,\dots,K\}\setminus\{r_1\}$. 
}

\begin{proof}
PoB can be represented by
\begin{equation}
S_{PoB}(r_1)=\mathbb{P}(Y_{r_1}>Y_{r_2},Y_{r_1}>Y_{r_3},\dots,Y_{r_1}>Y_{r_K}),
\end{equation}
where $(r_2,\dots,r_K)$ is an arbitrary permutation of elements in $\{1,2,\dots,K\}\setminus\{r_1\}$. 
Thus, from Lemma \ref{theo2}, it is identifiable by
\begin{align}
&S_{PoB}(r_1)\\
&=\mathbb{P}(Y_{r_1}>Y_{r_2},Y_{r_1}>Y_{r_3},\dots,Y_{r_1}>Y_{r_K})\\
&=\mathbb{P}(Y_{r_1}>F_{Y|X={r_2}}^{-1}(F_{Y|X={r_1}}(Y_{r_1})),Y_{r_1}>F_{Y|X={r_3}}^{-1}(F_{Y|X={r_1}}(Y_{r_1})),\dots,Y_{r_1}>F_{Y|X={r_K}}^{-1}(F_{Y|X={r_1}}(Y_{r_1})))\\
&=\mathbb{P}(Y>F_{Y|X={r_2}}^{-1}(F_{Y|X={r_1}}(Y)),Y>F_{Y|X={r_3}}^{-1}(F_{Y|X={r_1}}(Y)),\dots,Y>F_{Y|X={r_K}}^{-1}(F_{Y|X={r_1}}(Y))|X=r_1).
\end{align}
\end{proof}

\noindent{\bf Theorem \ref{theorem3}.}
{\it 
Under SCM ${\cal M}$ and Assumptions \ref{CON}, \ref{ASEXO2}, and \ref{PODm}, 
$\hat{\sigma}(\bm{r})$ and $\hat{\eta}(r_1)$ converge in probability to 
$S_{PoR}(\bm{r})$ and $S_{PoB}(r_1)$, respectively, at the rate ${\cal O}_p(\sum_{k=1}^K{{N^{(k)}}^{-1/2}})$.
}

\begin{proof}
The error of $\hat{\sigma}(\bm{r})$ is given as
\begin{align}
&\sigma(\bm{r})-\hat{\sigma}(\bm{r})\\
&=\mathbb{P}\Big(Y> F_{Y|X={r_2}}^{-1}(F_{Y|X={r_1}}(Y))> \dots> F_{Y|X={r_K}}^{-1}(F_{Y|X={r_1}}(Y))\Big|X={r_1}\Big)\\
&-\frac{1}{N^{(r_1)}}\sum_{i=1}^{N^{(r_1)}}\mathbb{I}\Bigg(Y_i^{(r_1)}> Y^{(r_2)}_{\left(\left\lfloor N^{(r_2)}\hat{F}_{Y|X=r_2}(Y_i^{(r_1)}) \right\rfloor+1\right)}> \dots> Y^{(r_K)}_{\left(\left\lfloor N^{(r_K)}\hat{F}_{Y|X=r_K}(Y_i^{(r_1)}) \right\rfloor+1\right)}\Bigg)\\
&=\Bigg\{\mathbb{P}\Big(Y> F_{Y|X={r_2}}^{-1}(F_{Y|X={r_1}}(Y))> \dots> F_{Y|X={r_K}}^{-1}(F_{Y|X={r_1}}(Y))\Big|X={r_1}\Big)\\
&-\frac{1}{N^{(r_1)}}\sum_{i=1}^{N^{(r_1)}}\mathbb{I}\Big(Y> F_{Y|X={r_2}}^{-1}(F_{Y|X={r_1}}(Y_i^{(r_1)}))> \dots> F_{Y|X={r_K}}^{-1}(F_{Y|X={r_1}}(Y_i^{(r_1)}))\Big)\Bigg\} \cdots\text{(A)}\\
&+\Bigg\{\frac{1}{N^{(r_1)}}\sum_{i=1}^{N^{(r_1)}}\mathbb{I}\Big(Y> F_{Y|X={r_2}}^{-1}(F_{Y|X={r_1}}(Y_i^{(r_1)}))> \dots> F_{Y|X={r_K}}^{-1}(F_{Y|X={r_1}}(Y_i^{(r_1)}))\Big)\\
&-\frac{1}{N^{(r_1)}}\sum_{i=1}^{N^{(r_1)}}\mathbb{I}\Bigg(Y_i^{(r_1)}> Y^{(r_2)}_{\left(\left\lfloor N^{(r_2)}\hat{F}_{Y|X=r_2}(Y_i^{(r_1)}) \right\rfloor+1\right)}> \dots> Y^{(r_K)}_{\left(\left\lfloor N^{(r_K)}\hat{F}_{Y|X=r_K}(Y_i^{(r_1)}) \right\rfloor+1\right)}\Bigg)\Bigg\} \cdots\text{(B)}.
\end{align}
For term (A), since $Y_i^{(r_1)}$ are i.i.d. samples from $\mathbb{P}(Y|X={r_1})$, $\frac{1}{N^{(r_1)}}\sum_{i=1}^{N^{(r_1)}}\mathbb{I}\Big(Y> F_{Y|X={r_2}}^{-1}(F_{Y|X={r_1}}(Y_i^{(r_1)}))> \dots> F_{Y|X={r_K}}^{-1}(F_{Y|X={r_1}}(Y_i^{(r_1)}))\Big) \rightarrow_p \mathbb{P}\Big(Y> F_{Y|X={r_2}}^{-1}(F_{Y|X={r_1}}(Y))> \dots> F_{Y|X={r_K}}^{-1}(F_{Y|X={r_1}}(Y))|X={r_1}\Big)$ at the rate ${\cal O}_p({{N^{(r_1)}}^{-1/2}})$ as $N^{(r_1)} \rightarrow \infty$.
For term (B), since no ties $Y> F_{Y|X={r_2}}^{-1}(F_{Y|X={r_1}}(Y))> \dots> F_{Y|X={r_K}}^{-1}(F_{Y|X={r_1}}(Y))$ and $Y^{(r_k)}_{\left(\left\lfloor N^{(r_k)}\hat{F}_{Y|X=r_k}(Y_i^{(r_1)}) \right\rfloor+1\right)} \rightarrow_p F_{Y|X={r_k}}^{-1}(F_{Y|X={r_1}}(Y_i^{(r_1)}))$ at the rate ${\cal O}_p({{N^{(r_1)}}^{-1/2}})$ as $N^{(k)} \rightarrow \infty$ for any $k=1,\dots,K$ \citep{Vaart1998}, from the delta method, we have $\frac{1}{N^{(r_1)}}\sum_{i=1}^{N^{(r_1)}}\mathbb{I}\Bigg(Y_i^{(r_1)}> Y^{(r_2)}_{\left(\left\lfloor N^{(r_2)}\hat{F}_{Y|X=r_2}(Y_i^{(r_1)}) \right\rfloor+1\right)}> \dots> Y^{(r_K)}_{\left(\left\lfloor N^{(r_K)}\hat{F}_{Y|X=r_K}(Y_i^{(r_1)}) \right\rfloor+1\right)}\Bigg) \rightarrow_p \frac{1}{N^{(r_1)}}\sum_{i=1}^{N^{(r_1)}}\mathbb{I}\Big(Y> F_{Y|X={r_2}}^{-1}(F_{Y|X={r_1}}(Y_i^{(r_1)}))> \dots> F_{Y|X={r_K}}^{-1}(F_{Y|X={r_1}}(Y_i^{(r_1)}))\Big)$ at the rate ${\cal O}_p({{N^{(r_k)}}^{-1/2}})$ as $N^{(k)} \rightarrow \infty$ for each $k=2,\dots,K$.
Then, $\hat{\sigma}(\bm{r})$ converges in probability to $S_{PoR}(\bm{r})$ at the rate ${\cal O}_p(\sum_{k=1}^K{{N^{(k)}}^{-1/2}})$.
Similarly,
the error of $\hat{\eta}(r_1)$ is given as
\begin{align}
&\eta(r_1)-\hat{\eta}(r_1)\\
&=\mathbb{P}\Big(Y> F_{Y|X={r_2}}^{-1}(F_{Y|X={r_1}}(Y)), \dots, Y> F_{Y|X={r_K}}^{-1}(F_{Y|X={r_1}}(Y))\Big|X={r_1}\Big)\\
&-\frac{1}{N^{(r_1)}}\sum_{i=1}^{N^{(r_1)}}\mathbb{I}\Bigg(Y_i^{(r_1)}> Y^{(r_2)}_{\left(\left\lfloor N^{(r_2)}\hat{F}_{Y|X=r_2}(Y_i^{(r_1)}) \right\rfloor+1\right)}, \dots,Y_i^{(r_1)}> Y^{(r_K)}_{\left(\left\lfloor N^{(r_K)}\hat{F}_{Y|X=r_K}(Y_i^{(r_1)}) \right\rfloor+1\right)}\Bigg)\\
&=\Bigg\{\mathbb{P}\Big(Y> F_{Y|X={r_2}}^{-1}(F_{Y|X={r_1}}(Y)), \dots, Y> F_{Y|X={r_K}}^{-1}(F_{Y|X={r_1}}(Y))\Big|X={r_1}\Big)\\
&-\frac{1}{N^{(r_1)}}\sum_{i=1}^{N^{(r_1)}}\mathbb{I}\Big(Y> F_{Y|X={r_2}}^{-1}(F_{Y|X={r_1}}(Y_i^{(r_1)})), \dots, Y> F_{Y|X={r_K}}^{-1}(F_{Y|X={r_1}}(Y_i^{(r_1)}))\Big)\Bigg\} \cdots\text{(A)}\\
&+\Bigg\{\frac{1}{N^{(r_1)}}\sum_{i=1}^{N^{(r_1)}}\mathbb{I}\Big(Y> F_{Y|X={r_2}}^{-1}(F_{Y|X={r_1}}(Y_i^{(r_1)})), \dots, Y> F_{Y|X={r_K}}^{-1}(F_{Y|X={r_1}}(Y_i^{(r_1)}))\Big)\\
&-\frac{1}{N^{(r_1)}}\sum_{i=1}^{N^{(r_1)}}\mathbb{I}\Bigg(Y_i^{(r_1)}> Y^{(r_2)}_{\left(\left\lfloor N^{(r_2)}\hat{F}_{Y|X=r_2}(Y_i^{(r_1)}) \right\rfloor+1\right)}, \dots, Y_i^{(r_1)}> Y^{(r_K)}_{\left(\left\lfloor N^{(r_K)}\hat{F}_{Y|X=r_K}(Y_i^{(r_1)}) \right\rfloor+1\right)}\Bigg)\Bigg\} \cdots\text{(B)}.
\end{align}
For term (A), since $Y_i^{(r_1)}$ are i.i.d. samples from $\mathbb{P}(Y|X={r_1})$, $\frac{1}{N^{(r_1)}}\sum_{i=1}^{N^{(r_1)}}\mathbb{I}\Big(Y> F_{Y|X={r_2}}^{-1}(F_{Y|X={r_1}}(Y_i^{(r_1)})), \dots,Y> F_{Y|X={r_K}}^{-1}(F_{Y|X={r_1}}(Y_i^{(r_1)}))\Big) \rightarrow_p \mathbb{P}\Big(Y> F_{Y|X={r_2}}^{-1}(F_{Y|X={r_1}}(Y)), \dots, Y> F_{Y|X={r_K}}^{-1}(F_{Y|X={r_1}}(Y))|X={r_1}\Big)$ at the rate ${\cal O}_p({{N^{(r_1)}}^{-1/2}})$ as $N^{(r_1)} \rightarrow \infty$.
For term (B), since no ties $Y> F_{Y|X={r_2}}^{-1}(F_{Y|X={r_1}}(Y)), \dots, Y> F_{Y|X={r_K}}^{-1}(F_{Y|X={r_1}}(Y))$ and $Y^{(r_k)}_{\left(\left\lfloor N^{(r_k)}\hat{F}_{Y|X=r_k}(Y_i^{(r_1)}) \right\rfloor+1\right)} \rightarrow_p F_{Y|X={r_k}}^{-1}(F_{Y|X={r_1}}(Y_i^{(r_1)}))$ at the rate ${\cal O}_p({{N^{(r_1)}}^{-1/2}})$ as $N^{(k)} \rightarrow \infty$ for any $k=1,\dots,K$ \citep{Vaart1998}, from the delta method, we have $\frac{1}{N^{(r_1)}}\sum_{i=1}^{N^{(r_1)}}\mathbb{I}\Bigg(Y_i^{(r_1)}> Y^{(r_2)}_{\left(\left\lfloor N^{(r_2)}\hat{F}_{Y|X=r_2}(Y_i^{(r_1)}) \right\rfloor+1\right)}, \dots, Y_i^{(r_1)}> Y^{(r_K)}_{\left(\left\lfloor N^{(r_K)}\hat{F}_{Y|X=r_K}(Y_i^{(r_1)}) \right\rfloor+1\right)}\Bigg) \rightarrow_p \frac{1}{N^{(r_1)}}\sum_{i=1}^{N^{(r_1)}}\mathbb{I}\Big(Y> F_{Y|X={r_2}}^{-1}(F_{Y|X={r_1}}(Y_i^{(r_1)})), \dots, Y> F_{Y|X={r_K}}^{-1}(F_{Y|X={r_1}}(Y_i^{(r_1)}))\Big)$ at the rate ${\cal O}_p({{N^{(r_k)}}^{-1/2}})$ as $N^{(k)} \rightarrow \infty$ for each $k=2,\dots,K$.
Then, $\hat{\eta}(r_1)$ converges in probability to $S_{PoB}(r_1)$ at the rate ${\cal O}_p(\sum_{k=1}^K{{N^{(k)}}^{-1/2}})$.
\end{proof}

\noindent{\bf Theorem \ref{theoremb_1}.}
{\it
Under SCM ${\cal M}$ and Assumptions \ref{CON} and \ref{ASEXO2},
for any $\bm{r} \in \mathfrak{R}$, 
$S_{PoR}(\bm{r})$ is bounded by $\sigma_L(\bm{r})\leq S_{PoR}(\bm{r}) \leq \sigma_U(\bm{r})$, where
\begin{align}
&\sigma_L(\bm{r})=\max\left\{\sum_{k=1}^{K-1}\sup_y\{F_{Y|X=r_{k+1}}(y)-F_{Y|X=r_{k}}(y)\}-K+2,0\right\},\\
&\sigma_U(\bm{r})= \min_{k=1,\dots,K-1}\left\{1-\sup_y\{F_{Y|X=r_k}(y)-F_{Y|X=r_{k+1}}(y)\}\right\}.
\end{align}
}

\begin{proof}
Since $S_{PoR}(\bm{r})$ is given as
\begin{equation}
\begin{aligned}
&\mathbb{P}\left(Y_{r_1}> Y_{r_2}> \dots> Y_{r_K}\right)=\mathbb{P}\left((Y_{r_1}> Y_{r_2})\land(Y_{r_2}>Y_{r_3})\land \dots(Y_{r_{K-1}}> Y_{r_K})\right),
\end{aligned}    
\end{equation}
$S_{PoR}(\bm{r})$ is bounded as follows from the Fr\'{e}chet inequalities \citep{Frechet1935,Frechet1960}:
\begin{equation}
\begin{aligned}
&\max\left\{\sum_{k=1}^{K-1}\mathbb{P}(Y_{r_{k}}> Y_{r_{k+1}})-K+2,0\right\}\leq S_{PoR}(\bm{r})\leq \min_{k=1,\dots,K-1}\left\{\mathbb{P}(Y_{r_{k}}> Y_{r_{k+1}})\right\}.
\end{aligned}    
\end{equation}
Then, we have the following results using the bounds of $\psi$ given in Eqs.~\eqref{eq7} and \eqref{eq8}:  
we have
\begin{align}
&\max\left\{\sum_{k=1}^{K-1}\sup_y\{\mathbb{P}(Y_{r_{k+1}}> y)-\mathbb{P}(Y_{r_{k}}> y)\}-K+2,0\right\}\\
&\leq S_{PoR}(\bm{r})\\
&\leq \min_{k=1,\dots,K-1}\left\{1-\sup_y\{\mathbb{P}(Y_{r_{k}}> y)-\mathbb{P}(Y_{r_{k+1}}> y)\}\right\}.
\end{align}
Then, from Assumption \ref{ASEXO2}, we have
\begin{align}
&\max\left\{\sum_{k=1}^{K-1}\sup_y\{F_{Y|X=r_{k+1}}(y)-F_{Y|X=r_{k}}(y)\}-K+2,0\right\}\\
&\leq S_{PoR}(\bm{r})\\
&\leq \min_{k=1,\dots,K-1}\Big\{1-\sup_y\{F_{Y|X=r_k}(y)-F_{Y|X=r_{k+1}}(y)\}\big\}.
\end{align}
\end{proof}

\noindent{\bf Theorem \ref{theoremb_2}.}
{\it
Under SCM ${\cal M}$ and Assumptions \ref{CON} and \ref{ASEXO2},
for any $r_1=1,2,\dots,K$,
$S_{PoB}(r_1)$ is bounded by $\eta_L(r_1)\leq S_{PoB}(r_1) \leq \eta_U(r_1)$, where
\begin{align}
&\eta_L(r_1)=\max\left\{\sum_{k=1}^{K-1}\sup_y\{F_{Y|X=r_{k+1}}(y)-F_{Y|X=r_{1}}(y)\}-K+2,0\right\},\\
&\eta_U(r_1)= \min_{k=1,\dots,K-1}\left\{1-\sup_y\{F_{Y|X=r_1}(y)-F_{Y|X=r_{k+1}}(y)\}\right\}.
\end{align}
}

\begin{proof}
Similarly, since $S_{PoB}(r_1)$ is given as
\begin{align}
\mathbb{P}\left(Y_{r_1}> Y_{r_2},Y_{r_1}> Y_{r_3},\dots, Y_{r_1}> Y_{r_K}\right)=\mathbb{P}\left((Y_{r_1}> Y_{r_2})\land(Y_{r_1}>Y_{r_3})\land \dots(Y_{r_1}> Y_{r_K})\right),
\end{align}
$S_{PoB}(r_1)$ is bounded as follows from the Fr\'{e}chet inequalities:
\begin{equation}
\begin{aligned}
&\max\left\{\sum_{k=1}^{K-1}\mathbb{P}(Y_{r_{1}}> Y_{r_{k+1}})-K+2,0\right\}\leq S_{PoB}(r_1)\leq \min_{k=1,\dots,K-1}\left\{\mathbb{P}(Y_{r_{1}}> Y_{r_{k+1}})\right\}.
\end{aligned}    
\end{equation}
Then, we have the following results: 
\begin{align}
&\max\left\{\sum_{k=1}^{K-1}\sup_y\{\mathbb{P}(Y_{r_{k+1}}> y)-\mathbb{P}(Y_{r_{1}}> y)\},0\right\}\\
&\leq S_{PoB}(r_1)\\
&\leq \min_{k=1,\dots,K-1}\left\{1-\sup_y\{\mathbb{P}(Y_{r_{1}}> y)-\mathbb{P}(Y_{r_{k+1}}> y)\}\right\}.
\end{align}
Then, from Assumption \ref{ASEXO2}, we have
\begin{align}
&\max\left\{\sum_{k=1}^{K-1}\sup_y\{F_{Y|X=r_{k+1}}(y)-F_{Y|X=r_{1}}(y)\}-K+2,0\right\}\\
&\leq S_{PoB}(r_1)\\
&\leq \min_{k=1,\dots,K-1}\left\{1-\sup_y\{F_{Y|X=r_1}(y)-F_{Y|X=r_{k+1}}(y)\}\right\}.
\end{align}
\end{proof}

\noindent{\bf Theorem \ref{theorem4}.}
{\it 
Under SCM ${\cal M}$ and Assumptions \ref{CON} and \ref{ASEXO2}, 
if $\Omega_Y$ is bounded,
$\hat{\sigma}_L(\bm{r})$, $\hat{\sigma}_U(\bm{r})$, $\hat{\eta}_L(r_1)$, and $\hat{\eta}_U(r_1)$ converge in probability to $\sigma_L(\bm{r})$, $\sigma_U(\bm{r})$, $\eta_L(r_1)$, and $\eta_U(r_1)$, respectively, at the rate ${\cal O}_p(1/M+\sum_{k=1}^K{{N^{(k)}}^{-1/2}})$.
}

\begin{proof}
The plugged-in estimates of
\begin{align}
&\sigma_L(\bm{r})=\max\left\{\sum_{k=1}^{K-1}\sup_y\{F_{Y|X=r_{k+1}}(y)-F_{Y|X=r_{k}}(y)\}-K+2,0\right\},\\
&\sigma_U(\bm{r})= \min_{k=1,\dots,K-1}\left\{1-\sup_y\{F_{Y|X=r_k}(y)-F_{Y|X=r_{k+1}}(y)\}\right\},\\
&\eta_L(r_1)=\max\left\{\sum_{k=1}^{K-1}\sup_y\{F_{Y|X=r_{k+1}}(y)-F_{Y|X=r_{1}}(y)\}-K+2,0\right\},\\
&\eta_U(r_1)= \min_{k=1,\dots,K-1}\left\{1-\sup_y\{F_{Y|X=r_1}(y)-F_{Y|X=r_{k+1}}(y)\}\right\}
\end{align}
are
\begin{align}
&\hat{\sigma}_L(\bm{r})=\max\left\{\sum_{k=1}^{K-1}\hat{\sup}_y\{\hat{F}_{Y|X=r_{k+1}}(y)-\hat{F}_{Y|X=r_{k}}(y)\}-K+2,0\right\},\\
&\hat{\sigma}_U(\bm{r})= \min_{k=1,\dots,K-1}\left\{1-\hat{\sup}_y\{\hat{F}_{Y|X=r_k}(y)-\hat{F}_{Y|X=r_{k+1}}(y)\}\right\},\\
&\hat{\eta}_L(r_1)=\max\left\{\sum_{k=1}^{K-1}\hat{\sup}_y\{\hat{F}_{Y|X=r_{k+1}}(y)-\hat{F}_{Y|X=r_{1}}(y)\}-K+2,0\right\},\\
&\hat{\eta}_U(r_1)= \min_{k=1,\dots,K-1}\left\{1-\hat{\sup}_y\{\hat{F}_{Y|X=r_1}(y)-\hat{F}_{Y|X=r_{k+1}}(y)\}\right\}.
\end{align}
Since they are all Lipschitz continuous for each $\hat{F}_{Y|X=r_1}(y)$ ($k=1,\dots,k$) when $\Omega_Y$ is bounded, then they converge in probability at the rate ${\cal O}_p(\sum_{k=1}^K{{N^{(k)}}^{-1/2}})$ to
\begin{align}
&\tilde{\sigma}_L(\bm{r})=\max\left\{\sum_{k=1}^{K-1}\hat{\sup}_y\{F_{Y|X=r_{k+1}}(y)-F_{Y|X=r_{k}}(y)\}-K+2,0\right\},\\
&\tilde{\sigma}_U(\bm{r})= \min_{k=1,\dots,K-1}\left\{1-\hat{\sup}_y\{F_{Y|X=r_k}(y)-F_{Y|X=r_{k+1}}(y)\}\right\},\\
&\tilde{\eta}_L(r_1)=\max\left\{\sum_{k=1}^{K-1}\hat{\sup}_y\{F_{Y|X=r_{k+1}}(y)-F_{Y|X=r_{1}}(y)\}-K+2,0\right\},\\
&\tilde{\eta}_U(r_1)= \min_{k=1,\dots,K-1}\left\{1-\hat{\sup}_y\{F_{Y|X=r_1}(y)-F_{Y|X=r_{k+1}}(y)\}\right\}.
\end{align}
Since $\tilde{\sigma}_L(\bm{r})$, $\tilde{\sigma}_U(\bm{r})$, $\tilde{\eta}_L(r_1)$, and $\tilde{\eta}_U(r_1)$ are Lipschitz continuous w.r.t. $\hat{\sup}_{y}$ functions, 
if $\Omega_Y$ is bounded,
they converge to $\sigma_L(\bm{r})$, $\sigma_U(\bm{r})$, $\eta_L(r_1)$, and $\eta_U(r_1)$  at the rate ${\cal O}(1/M)$.
In total, $\hat{\sigma}_L(\bm{r})$, $\hat{\sigma}_U(\bm{r})$, $\hat{\eta}_L(r_1)$, and $\hat{\eta}_U(r_1)$ converge in probability to $\sigma_L(\bm{r})$, $\sigma_U(\bm{r})$, $\eta_L(r_1)$, and $\eta_U(r_1)$, respectively, at the rate ${\cal O}_p(1/M+\sum_{k=1}^K{{N^{(k)}}^{-1/2}})$.
\end{proof}

\section*{Appendix B: Additional Information about the Numerical Experiment}
\label{appB1}

We present curves of absolute errors of estimates vs. $K$ in Figures~\ref{fig:PoRK} and~\ref{fig:PoBK}.  
The experiments are performed using an Apple M1 (16GB).

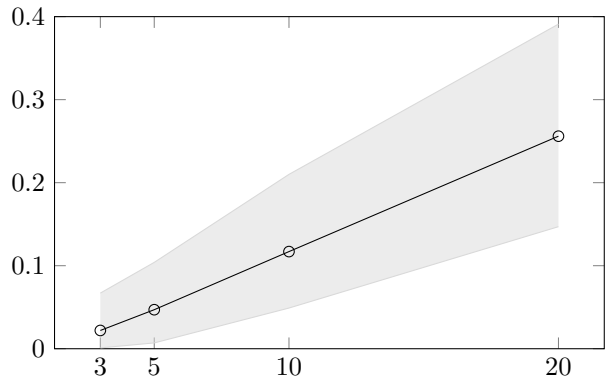
\begin{figure}[H]
\centering
\begin{tikzpicture}
\begin{axis}[
    width=0.5\linewidth,
    height=6cm,
    ymin=0, ymax=0.4,
    xtick={3,5,10,20},
    xticklabels={$3$, $5$, $10$, $20$},
    legend style={font=\small, at={(0.05,0.05)}, anchor=south west}
]

\path [draw=none, fill=gray!30, opacity=0.5] 
(3,0.067) -- (5,0.104) -- (10,0.210) -- (20,0.391) -- 
(20,0.147) -- (10,0.049) -- (5,0.007) -- (3,0.001) -- cycle;

\addplot[name path=mean,
draw=black,
mark=o, 
] coordinates {
    (3, 0.022)
    (5, 0.047)
    (10, 0.117)
    (20, 0.256)
};

\addplot[name path=lower, draw=gray!30] coordinates {
    (3, 0.001)
    (5, 0.007)
    (10, 0.049)
    (20, 0.147)
};

\addplot[name path=upper, draw=gray!30] coordinates {
    (3, 0.067)
    (5, 0.104)
    (10, 0.210)
    (20, 0.391)
};


\end{axis}
\end{tikzpicture}
\caption{Absolute errors of PoR estimates with 95\% CIs (y-axis) vs. $K$ (x-axis). 
}
\label{fig:PoRK}
\end{figure}

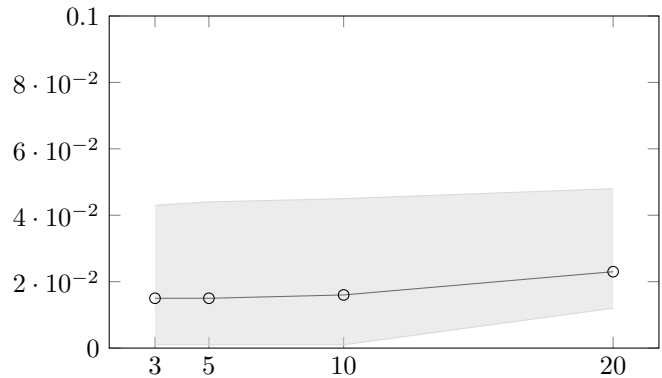
\begin{figure}[H]
\centering
\begin{tikzpicture}
\begin{axis}[
    width=0.5\linewidth,
    height=6cm,
    ymin=0, ymax=0.1,
    xtick={3,5,10,20},
    xticklabels={$3$, $5$, $10$, $20$},
    error bars/y dir=both,
    error bars/y explicit,
    legend style={draw=none, font=\small, at={(0.05,0.05)}, anchor=south west}
]
\addplot[
    black,
    mark=o, 
] coordinates {
    (3, 0.015)
    (5, 0.015)
    (10, 0.016)
    (20, 0.023)
};

\addplot[name path=lower, draw=gray!30] coordinates {
    (3, 0.001)
    (5, 0.001)
    (10, 0.001)
    (20, 0.012)
};

\addplot[name path=upper, draw=gray!30] coordinates {
    (3, 0.043)
    (5, 0.044)
    (10, 0.045)
    (20, 0.048)
};

\path [draw=none, fill=gray!30, opacity=0.5] 
(3,0.043) -- (5,0.044) -- (10,0.045) -- (20,0.048) -- 
(20,0.012) -- (10,0.001) -- (5,0.001) -- (3,0.001) -- cycle;

\addplot[dashed, thick] coordinates {(3, 0.666) (20, 0.666)};
\end{axis}
\end{tikzpicture}
\caption{Absolute errors of PoB estimates with 95\% CIs (y-axis) vs. $K$ (x-axis). 
}
\label{fig:PoBK}
\end{figure}

\section*{Appendix C: Additional Information about the Application to Real-world}
\label{appB}

We provide additional information about the application in the body of the paper.
Figure \ref{tab:addapp1} shows the mean estimates with 95\% CIs of the upper bounds (UB) and the lower bound (LB) of  PoR or PoB metrics.

\begin{table}[H]
    \centering
    \scalebox{0.95}{
    \begin{tabular}{c|c}
    \hline
    Bounds of  PoR or PoB &  Estimates\\
        \hline\hline
      \hline
     UB of $S_{PoR}(B,H,S)$  & 0.582 [0.333,0.788] \\
     LB of $S_{PoR}(B,H,S)$  & 0.000 [0.000,0.000] \\
     UB of $S_{PoR}(B,S,H)$  & 0.541 [0.257,0.788] \\
     LB of $S_{PoR}(B,S,H)$  & 0.000 [0.000,0.000] \\
     UB of $S_{PoR}(H,B,S)$  & 0.564 [0.250,0.833] \\
     LB of $S_{PoR}(H,B,S)$  & 0.000 [0.000,0.000] \\
     UB of $S_{PoR}(H,S,B)$  & 0.719 [0.417,0.917] \\
     LB of $S_{PoR}(H,S,B)$  & 0.033 [0.000,0.303] \\
     UB of $S_{PoR}(S,B,H)$  & 0.618 [0.333,0.833] \\
     LB of $S_{PoR}(S,B,H)$  & 0.000 [0.000,0.000] \\
     UB of $S_{PoR}(S,H,B)$  & 0.742 [0.431,0.916] \\
     LB of $S_{PoR}(S,H,B)$  & 0.038 [0.000,0.333] \\
     \hline
     UB of $S_{PoB}(B)$  & 0.618 [0.272,0.909]\\
     LB of $S_{PoB}(B)$  & 0.000 [0.000,0.000]\\
     UB of $S_{PoB}(H)$  & 0.832 [0.583,1.000]\\
     LB of $S_{PoB}(H)$  & 0.000 [0.000,0.000]\\
     UB of $S_{PoB}(S)$  & 0.616 [0.333,0.917]\\
     LB of $S_{PoB}(S)$  & 0.000 [0.000,0.000]\\
      \hline
    \end{tabular}
    }
    \caption{Mean estimates with 95\% CIs of the upper bounds (UB) and the lower bound (LB) of  PoR or PoB metrics.
    }
    \label{tab:addapp1}
\end{table}

\end{document}